\documentclass[letterpaper]{article}
\usepackage{uai2020}
\usepackage[margin=1in]{geometry}

\usepackage{microtype}
\usepackage{graphicx}
\usepackage{subcaption}
\usepackage{amsfonts}
\usepackage{amsmath}
\usepackage{amsthm}
\usepackage{mathtools}
\usepackage{booktabs} 
\usepackage{hyperref}
\usepackage{xcolor}
\usepackage{times}
\usepackage[backend=bibtex,giveninits=true,style=authoryear]{biblatex}
\newtheorem{theorem}{Theorem}
\newtheorem{corollary}{Corollary}
\newtheorem{definition}{Definition}

\newcommand{\SumNode}{\mathsf{S}}
\newcommand{\ProductNode}{\mathsf{P}}
\newcommand{\TransformationNode}{\mathsf{T}}
\newcommand{\LeafNode}{\mathsf{L}}
\newcommand{\Node}{\mathsf{N}}

\newcommand{\cbar}{\,|\,}
\newcommand{\scope}{\ensuremath{\psi}}
\newcommand{\ch}{\ensuremath{\mathbf{ch}}}
\newcommand{\graph}{\mathcal{G}}

\title{Sum-Product-Transform Networks: Exploiting Symmetries using Invertible Transformations}
\author{ \bf{Tom\'a\v s Pevn\'y\thanks{Tom\'a\v s Pevn\'y is also with Avast Software s.r.o},  V\'aclav \v Sm\'idl}\\
Artificial Intelligence Center\\
Czech Technical University\\
Prague, Czech republic\\
\And
{\bf Martin Trapp}  \\
Graz University of Technology \\
Graz, Austria \\
\And
\bf{Ond\v{r}ej Pol\'{a}\v{c}ek, Tom\'a\v s Oberhuber} \\
Faculty of Nuclear Sciences \\and Physical Engineering\\
Czech Technical University\\
Prague, Czech republic\\
}
\begin{document}
\maketitle

\newcommand{\sumnode}{\ensuremath{\Pi}}
\newcommand{\productnode}{\ensuremath{\Sigma}}
\newcommand{\densenode}{\ensuremath{\Phi}}
\newcommand{\setindex}[1]{\ensuremath{\mathbf{#1}}}

\begin{abstract}
In this work, we propose\textbf{} Sum-Product-Transform Networks (SPTN), an extension of sum-product networks that uses invertible transformations as additional internal nodes. 
The type and placement of transformations determine properties of the resulting SPTN with many interesting special cases. 
Importantly, SPTN with Gaussian leaves and affine transformations pose the same inference task tractable that can be computed efficiently in SPNs. 
We propose to store affine transformations in their SVD decompositions using an efficient parametrization of unitary matrices by a set of Givens rotations. 
Last but not least, we demonstrate that G-SPTNs achieve state-of-the-art results on the density estimation task and are competitive with state-of-the-art methods for anomaly detection. 

\end{abstract}

\section{INTRODUCTION}\label{sec:motivation}

Modeling and manipulating complex joint probability distributions are central goals in machine learning.
Its importance derives from the fact that probabilistic models can be understood as multi-purpose tools, allowing them to solve many machine learning tasks using probabilistic inference.
However, recent flexible and expressive techniques for density estimation, such as normalizing flows~\parencite{Rezende2015Flows,kobyzev2019normalizing} and neural auto-regressive density estimators~\parencite{Uria2016NADE}, lack behind when it comes to performing inference tasks efficiently.
Motivated by the absence of tractable probabilistic inference capabilities, recent work in probabilistic machine learning has put forth many instances of so-called Probabilistic Circuits (PCs), such as Sum-Product Networks (SPNs)~\parencite{Poon11}, Probabilistic Sentential Decision Diagrams (PSDDs)~\parencite{kisa2014probabilistic} and Cutset network~\parencite{rahman2014cutset}.
In contrast to auto-regressive and flow-based techniques, PCs guarantee that many inference tasks can be computed exactly in time linear in their representation size.
The critical insights for PCs are that: i) high-dimensional probability distributions can be efficiently represented by composing convex combinations, factorizations, and tractable input distributions; and that ii) decomposability~\parencite{Darwiche2003BayesianNets} simplifies many inference scenarios to tractable inference at the input distributions.
Due to their favorable properties, PCs have been successfully applied for many complex machine learning tasks, e.g.~image segmentation~\parencite{Rathke2017,Stelzner2019}, semantic mapping~\parencite{Zheng2018}, and image classification~\parencite{Peharz2019RAT}.

\begin{figure}[t]
\centering{\includegraphics[width=\linewidth]{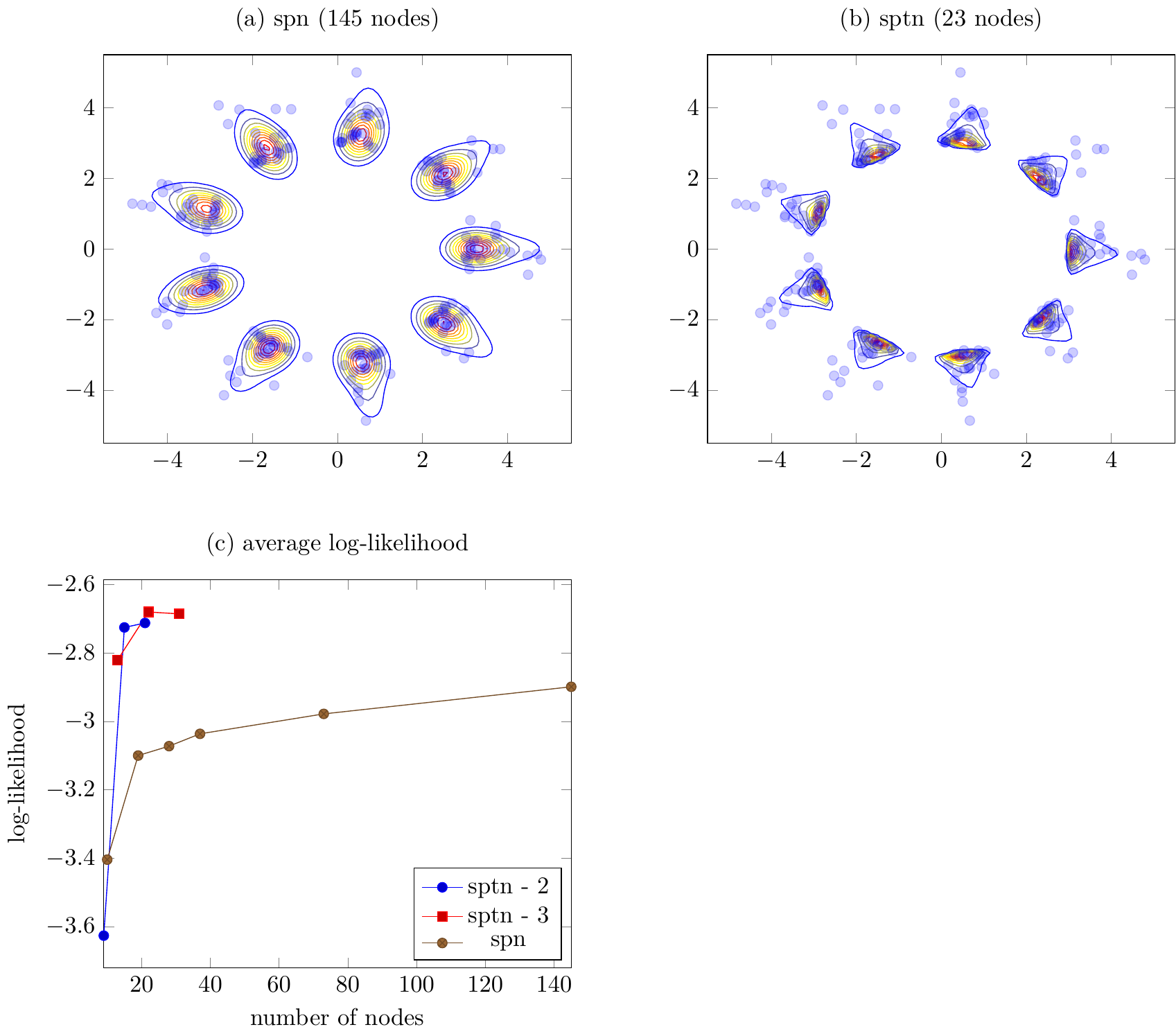}}
\caption{\label{fig:flower}Figures (a-b), respectively, show the density function of an SPN and SPTN overlayed onto a subset of training data. SPTNs can fit the data more effectively, by exploiting transformations of the density function and result in a more compact representation, c.f.~Figure (c).}
\end{figure}

To model complex probability distributions, PCs leverage a hierarchy of convex combinations and factorizations, resulting in a compact representation of an exponentially large mixture distribution.
However, by restricting to compositions of tractable input distributions using convex combinations and factorizations, PCs cannot exploit geometric properties, such as symmetries, in the density function and lack a compact representation in low-dimensional scenarios.
Thus, potentially resulting in inefficient representations of complex joint probability distributions in various scenarios.

In this paper, we propose to extend the compositions used in PCs to additionally include invertible transformations.
In particular, we introduce Sum-Product-Transformation Networks (SPTNs), which combine SPNs, i.e.~complete and decomposable PCs, with an additional change of variables transformations.  
The resulting model class naturally combines tractable computations in normalizing flows with tractable computations in SPNs.
SPTNs are an expressive and flexible probabilistic model that enables exploitation of the geometry, e.g.~symmetries, while facilitating tractable inference scenarios, depending on the network structure.

As an example, consider the \emph{flower} dataset illustrated in Figure~\ref{fig:flower}, which consists of nine petal leaves.
In this example, an SPN, as well as our model, was trained using $2\mathrm{D}$ Gaussian (full covariance) leaves. 
Naturally, one would leverage the fact that data distribution can be modeled efficiently by applying affine transformations, i.e.~using rotations around the origin.
However, in the case of SPNs, this cannot be exploited, and the dimensionality of the data naturally limits the depth of the model.
On the other hand, SPTNs can compactly represent even complex low-dimensional distributions as their depth is not limited by the dimensionality of the data distribution.
Further, SPTNs can exploit symmetries in the data, resulting in higher predictive performance, c.f.~Figure~\ref{fig:flower} (c).
In a variety of experiments, we show that SPTNs indeed achieve high predictive performance on various datasets and often outperform SPNs in terms of test log-likelihood.

Our main contributions can be summarised as follows:
\begin{itemize}
    \item We introduce an extension of probabilistic circuits (PCs) which interleaves common compositions in PCs with invertible transformations, resulting in a flexible tractable probabilistic model. This unifies two paradigms: probabilistic circuits and flow models into a single framework.
    \item We introduce a new affine flow, and conjecture that for many interesting applications an affine transformation is sufficient.
    Our affine flow has a native parametrization in SVD decomposition, which allows (i) efficient inverse and (ii) efficient calculation of determinant of the Jacobian and its gradients. 
    \item We introduce a tractable subclass, G-SPTNs, consisting of sum and product nodes, Gaussian leaves, and only affine transformations. G-SPTNs support efficient marginalization and computation of conditionals.
    \item Finally, we show on 21 benchmark datasets that G-SPTNs deliver better performance than SPNs, GMMs, and Masked Autoregressive Flows (MAF) in tasks of density estimation and anomaly detection.
\end{itemize}

\section{BACKGROUND}
Probabilistic circuits (PCs) are a large class of tractable probabilistic models that admit many probabilistic inference tasks in linear time (linear in their representation size).
\begin{definition}[Probabilistic Circuit]
Given a set of random variables (RVs) $\mathbf X$, a Probabilistic Circuit (PC) is defined as tuple $(\graph, \scope, \theta)$ consisting of a computational graph $\graph = (V,E)$, which is a directed acyclic graph (DAG) containing sum, product and leaf nodes, a scope-function $\scope: V \rightarrow 2^{\mathbf{X}}$ and a set of parameters $\theta$.
\end{definition}

In general, we additionally expect the \emph{scope-function} to fulfil the following properties: i) for all internal nodes $\Node \in V$ we have $\scope(\Node) = \bigcup_{\Node' \in \ch(\Node)} \scope(\Node')$ and ii) for each root node $\Node$, i.e.~each node without parents, we have $\scope(\Node) = \mathbf{X}$.
To guarantee many inference scenarios to be tractable, we additionally require the scope-function to fulfil that, for each product node $\ProductNode \in V$ the scopes of the children of $\ProductNode$ are disjoint, i.e.~$\bigcap_{\Node' \in \ch(\ProductNode)} \scope(\Node') = \emptyset$ (\emph{decomposability}).
In this paper, we further assume that, for each sum node $\SumNode \in V$ that $\scope(\Node) = \scope(\Node')\, \forall \Node, \Node' \in \ch(\SumNode)$ (\emph{completeness/smoothness)}.
Complete/smooth and decomposable PCs are often referred to as Sum-Product Networks (SPNs).

In an SPN, each leaf node $\LeafNode \in V$ is a (tractable) distribution over its scope $\scope(\Node)$, parametrized by $\theta_\LeafNode$.
Internal nodes, either compute a weighted sum with non-negative weights (sum node) of its children, i.e.~$\SumNode(x) = \sum_{\Node \in \ch(\SumNode)} w_{\SumNode,\Node} \Node(x)$ with $w_{\SumNode,\Node} \geq 0$, or compute a product of its children (product node), i.e.~$\ProductNode(x) = \prod_{\Node \in \ch(\ProductNode)} \Node(x_{\scope(\Node)})$, where $\ch(\Node)$ returns the set of children of node $\Node$.
Note that w.l.o.g.~we assume that all sum nodes are normalised, i.e.~$\sum_{\Node \in \ch(\SumNode)} w_{\SumNode,\Node} = 1$, c.f.~\parencite{Peharz2015}.

SPNs have recently gained increasing attention, due to their success in various applications, e.g.~\parencite{Stelzner2019,Peharz2019RAT}.  
Inspired by these successes, various flexible extensions of SPNs have recently been proposed, e.g.~SPNs over variational autoencoders (VAEs)~\parencite{tan2019hierarchical}, SPNs over Gaussian processes~\parencite{Trapp2020} and quotient nodes to represent conditional distributions within the SPN~\parencite{Sharir18}.
However, to the best of our knowledge, SPNs and PCs have not been extended to incorporate invertible transformations as of yet.

\section{SUM-PRODUCT- TRANSFORMATION~NETWORKS}

Sum-Product-Transformation Networks (SPTNs) naturally combine the taxonomy of SPNs with normalizing flows.
In SPTNs, we extend PCs to additionally include nodes representing a change of variables formulas. 

\begin{definition}[Sum-Product-Transformation Network]
A Sum-Product-Transformation Network (SPTN) over a set of RV $\mathbf{X}$ is an extension of PCs which is recursively defined as:
\begin{itemize}
    \item An arbitrary (tractable) input distribution is an SPTN (leaf node), i.e.~$\LeafNode(x) = p(x \cbar \theta_{\LeafNode})$.
    \item A product of SPTNs is an SPTN (product node), i.e.~$\ProductNode(x) = \prod_{\Node \in \ch(\ProductNode)} \Node(x_{\scope(\Node)})$.
    \item A convex combination of SPTNs is an SPTN (sum node), i.e.~$\SumNode(x) = \sum_{\Node \in \ch(\SumNode)} w_{\SumNode,\Node} \Node(x)$ with $w_{\SumNode,\Node} \geq 0$.
    \item An invertible transformation of an SPTN is an SPTN (transformation node), i.e.~$\TransformationNode(\Node(x)) = \Node(g(x))\det |J_{g}(x)|$ where $g(x)$ is a bijection and $J_{g}(x)$ denotes the Jacobian of the transformation.
\end{itemize}
\end{definition}

In the course of this paper, we will generally assume SPTNs to be complete/smooth and decomposable.
Note that those properties are akin to completeness and decomposability in SPNs, as transformation nodes ($\TransformationNode$) have only a single child and, thus, $\scope(\TransformationNode) = \scope(\Node) \forall \Node \in \ch(\TransformationNode)$.

\subsection{REALIZATION OF TRANSFORMATION~NODES}
To calculate the density function of a random variable $X$ after applying a transformation $f(x)$, requires that $f(x)$ is invertible, i.e.~it is a bijection.
For practical reasons, it is desired the determinant of the Jacobian of $f(x)$ to be efficiently calculated.
A recent technique that satisfies these properties is known as normalizing flows, we refer to~\parencite{papamakarios2019normalizing} for an overview, which either imposes a special structure on $f$ or relies on properties of the ODE equations.
We extend this family by introducing a variant of dense layers in feed-forward networks, which allows efficient inversion and computation of the Jacobian.

Feed-forward neural networks implement a function 
\begin{equation}
f(x) = \sigma(\mathbf{W}x + b) \, ,
\label{eq:dense}
\end{equation}
where $\mathbf{W}$ is a weight matrix, $b$ is a bias term, and $\sigma(x)$ is a (non-)linear transformation.
In the case of change of variables, $\mathbf{W}\in\mathbb{R}^{d,d}$ has to be a square matrix with full rank, due to the requirement of invertibility, and 
$b \in \mathbb{R}^d$. 
Since $\mathbf{W}$ has to be full rank, it can be expressed using a Singular Value Decomposition (SVD), i.e.~$\mathbf{W} = \mathbf{U}\mathbf{D}\mathbf{V}^{\top}$, where $\mathbf{U}$ and $\mathbf{V}$ are unitary matrices and $\mathbf{D}$ is a diagonal matrix.
SVD decompositions allow for a convenient calculation of the inverse of $f$ as 
\begin{equation}
    f^{-1}(z) =  \mathbf{V}\mathbf{D}^{-1}\mathbf{U}^{\top}(\sigma^{-1}(z) - b) \, .
\end{equation}
Further, we can conveniently calculate the logarithm of the determinant of the Jacobian as it holds that
\begin{equation}
    \log \left(\left|\frac{\partial f}{\partial x}\right| \right) = \sum_{i=1}^{d} \log |d_{ii}| + \sum_{i=1}^{d} \log \left|\frac{\partial \sigma_i}{\partial o_i}\right| \, ,
\end{equation}
where $o = \mathbf{U}\mathbf{D}\mathbf{V}^{\top}x + b$.

While the SVD decomposition has appealing properties, it is generally very expensive to calculate.
Therefore, we propose to store and optimize $\mathbf{W}$ in its SVD decomposition. 
This is possible if the group of unitary matrices $\mathcal{U}$ can be parametrized by $\theta \in \mathbb{R}^{\frac{1}{2}d(d-1)}$ such that (i) $\mathbf{U}(\theta) \in \mathbb{R}^{\mathrm{d,d}}$ is a unitary square matrix for arbitrary $\theta \in \mathbb{R}^{\frac{1}{2}d(d-1)}$ and (ii) for every unitary matrix $\mathbf{U}^{\prime}$ there exists $\theta^{\prime}$ such that $\mathbf{U}^{\prime} = \mathbf{U}(\theta^{\prime});$ and (iii) a gradient $\frac{\partial \mathbf{U}(\theta)}{\partial \theta}$ exists and can be computed efficiently.  
We discus two approaches to parametrize $\mathcal{U}$ under these conditions in Section~\ref{sec:unitaryMatrices}.

We want to emphasize that the parameters of transformation nodes can be shared within the model since they do not have any probabilistic interpretation. 
This allows a compact representation of transformation nodes within the model.
The type of transformation and their placements in the computational graph has an impact on the tractability of the resulting model. 
Therefore, we will now discuss a few important special cases:

\paragraph{Affine Gaussian SPTN (G-SPTN):}
SPTNs with Gaussian leaves and arbitrarily placed affine transformations can be transformed into an exponentially large mixture of Gaussians, c.f.~Theorem~\ref{thm:gmm}.
This has an important consequence as marginalization is now analytically tractable, which arises from the fact that affine-transformed Gaussian distributions remain Gaussian. 

\paragraph{Flow models:}
Any SPTN consisting only of transformation and product nodes is a flow. 
Note, however, that marginalization and computation of moments are generally not tractable.

\paragraph{SPN with Flexible Leaves:}
An SPTN with transformations only above the leaf nodes extends the set of possible leaf node distributions.  
Since the transformation is deferred only to leaves, in the univariate case, tractability is generally preserved.

Hence, it is possible to exploit tractability in certain parts of the model, while sacrificing it in favor of complex transformations in others. 
Allowing the practitioner to design the model as needed.

\begin{theorem}\label{thm:gmm}
Inference tasks that are tractable in SPNs are also tractable in SPTN with affine transformation nodes and Gaussian distribution at the leaves.
\end{theorem}
\begin{proof}
Let the SPTN be composed of sum, product, affine transformation, and Gaussian leaf nodes, i.e.~an G-SPTN.
Further, let us assume that all $\mu_{\cdot}$ are vectors and all $\mathbf{\Sigma}_{\cdot}$ are matrices of appropriate dimensions. 

Then, 
(i) \emph{An affine transformation of a Gaussian distributed vector is Gaussian.}
Specifically, let $x\sim \mathcal{N}(\mu_x, \mathbf{\Sigma}_x)$.  Then $y \sim \mathcal{N}(\mu_y, \mathbf{\Sigma}_y)$ with $\mu_y = \mathbf{W}\mu_x + b$ and $\mathbf{\Sigma}_y = W\mathbf{\Sigma}_xW^{\top}$.

(ii) \emph{The product distribution of Gaussian distributed vectors is Gaussian.}
Let $x_1\sim \mathcal{N}(\mu_1, \mathbf{\Sigma}_1)$ and $x_2\sim \mathcal{N}(\mu_2, \mathbf{\Sigma}_2)$.
Then,
\begin{equation}
    \begin{bmatrix} x_1 \\ x_2 \end{bmatrix} \sim \mathcal{N} \left(\begin{bmatrix} \mu_1 \\ \mu_2 \end{bmatrix}, \begin{bmatrix} \mathbf{\Sigma}_1 & 0 \\ 0 & \mathbf{\Sigma}_2 \end{bmatrix}\right) \, .
\end{equation}

(iii) \emph{The density function of any SPN can be represented by an exponentially large mixture.}
As shown in \cite{zhao2016collapsed,Trapp2019}, any SPN can be represented by an exponentially large mixture distribution over so-called induced trees.
The same applies to SPTNs as transformation nodes have only a single child, i.e., if a transformation node is included in the induced tree its child and the respective edge will also be included.

By applying (i)-(iii), we can express each component of the implicitly represented exponentially large mixture (of an SPTN) as a transformed Gaussian distribution with a block-diagonal covariance structure determined by the scope-function.
Even though, a Gaussian mixture model representation is not very useful per se, it shows that an SPTN with the affine transformations ``pulled down'' to the leaves is equivalent to the original SPTN.
Therefore, we can perform arbitrary marginalization tasks in a G-SPTN by ``pulling down'' the affine transformations and performing inference directly on the transformed leaves.
\end{proof}

\begin{corollary}
Marginal and conditional distributions of G-SPTN have the same analytical properties as SPNs using Gaussian distributions at the leaves with a block-diagonal covariance structure.
\end{corollary}

\paragraph{NODE SHARING}
SPTNs allow multiple ways of reducing the number of parameters via node sharing. Since the introduced transformation is just a new type of node, it can be shared in the computational graph just like the sum and product nodes.  This is illustrated in schematics, such as in Figure~\ref{fig:sharing}. 

\section{PARAMETERIZING UNITARY MATRICES}
\label{sec:unitaryMatrices}
We will now describe two methods to parametrize a group of unitary matrices, each having its advantages and disadvantages.\footnote{Both methods are implemented in a publicly available package \url{http://github.com}.}
\subsection{GIVENS PARAMETRIZATION}

The first parametrization relies on a set of Givens rotations. 
Let us assume a Givens rotation in $\mathbb{R}^{2\times2}$, parametrized by a single parameter $\theta\in\mathbb{R}$ as
\begin{equation}
\begin{bmatrix}
\hphantom{-} \cos(\theta)& \sin(\theta) \\
-\sin(\theta)& \cos(\theta) \\
\end{bmatrix} \, .
\end{equation}
 
For every value $\theta$ the above matrix is unitary and for every unitary matrix $\mathbf{U} \in \mathbb{R}^{2,2}$ with positive determinant there exists $\theta$ such that $\mathbf{U} = \mathbf{U}(\theta)$. 
As shown by \parencite{Polcari2014}, for any $d = 2^{k},$ $k > 1$ a group of unitary matrices in the space $\mathbb{R}^{d,d}$ is parametrized by a set of Givens transformations. 
Below, we will generalize this approach to arbitrary $d > 1$.

Let $\mathbf{G}^{r,s}(\theta)$ denote an almost diagonal matrix, with a Givens rotation on $r$ and $s$ columns.
As an example consider $\mathbf{G}^{1,3}(\theta)$ in $\mathbb{R}^{4,4}$, i.e.,
\begin{equation}
\mathbf{G}^{1,3}(\theta) = \begin{bmatrix}
\hphantom{-} \cos(\theta) & 0 & \sin(\theta) & 0  \\
0 & 1 & 0 & 0 \\
- \sin(\theta) & 0 & \cos(\theta) & 0  \\
0 & 0 & 0 & 1  \\
\end{bmatrix}.
\end{equation}

\begin{theorem}
Let $\mathbf{U} \in \mathbf{R}^{d,d},$ $d > 1$ be a unitary matrix. 
Then there exist $\{(\mathbf{G}^{r,s}(\theta_{r,s})) \vert 1\leq r < s \leq d, \theta_i \in \mathbf{R}\}, $ such that 
$$\mathbf{U} = \prod_{1 < r < s}^{d,d}\mathbf{G}^{r,s}(\theta_{r,s}) \, .$$
\end{theorem}
\begin{proof}
The proof is carried out by induction. 
For $d = 2$ the group of Unitary matrices consists of all rotations which coincides with $\{\mathbf{G}^{1,2}(\theta) | \theta \in \mathbb{R}\}\, .$
Let $\mathbf{U} \in \mathbf{R}^{d+1,d+1}$, and let $u$ denote the last column of the matrix $\mathbf{U}$. 
Since $\mathbf{U}$ is unitary, $\|u\|_{2} = 1$, and therefore it is a point on a $d+1$-dimensional sphere and can be expressed in polar coordinates~\parencite{Blumenson60} as
\begin{equation}
u = \begin{bmatrix*}[l]
\cos({\theta_{1,d+1}}) \\
\sin({\theta_{1,d+1}}) \cos({\theta_{2,d+1}}) \\
\sin({\theta_{1,d+1}}) \sin({\theta_{2,d+1}})\cos({\theta_{3,d+1}}) \\
\vdots\\
\sin({\theta_{1,d+1}}) \sin({\theta_{2,d+1}})\ldots \cos({\theta_{d,d+1}}) \\
\sin({\theta_{1,d+1}}) \sin({\theta_{2,d+1}})\ldots \sin({\theta_{d,d+1}}) \\
\end{bmatrix*}.
\end{equation}
Therefore $u^{\top}\mathbf{G}^{d,d+1}(\theta_{d,d+1})$ has the last, $(d+1)^{\mathrm{th}},$ coordinate equal to zero and $d^{\mathrm{th}}$ coordinate equal to $\sin({\theta_{1,d+1}}) \sin({\theta_{2,d+1}})\ldots\sin({\theta_{d-1,d+1})}$. 
Therefore if, 
\begin{align}
    \mathbf{G}^{d+1} = &\mathbf{G}^{d,d+1}(\theta_{d,d+1})\mathbf{G}^{d-1,d+1}(\theta_{d-1,d+1})\\
    &\ldots\mathbf{G}^{1,d+1}(\theta_{1,d+1}) \, ,
\end{align}
then $ u^{\top}\mathbf{G}^{d+1} = (1,0,0,\ldots,0)^{\top}.$ 
And because $\mathbf{U}$ is unitary, it holds that 
\begin{equation}
\mathbf{U}\mathbf{G}^{d+1} =
\begin{bmatrix}
0 & 1 \\
\hat{\mathbf{U}} & 0 \\
\end{bmatrix},   
\end{equation}
where $\hat{\mathbf{U}}$ is again unitary of dimension $\mathbf{R}^{d,d}$.
Therefore, an inductive assumption can be applied, which completes the proof.
\end{proof}
The corollary of this theorem is that 
$\prod_{1 < r < s}^{d,d}\mathbf{G}^{r,s}(\theta_{r,s}),$ parametrizes a whole group of positive definite unitary matrices in $\mathbb{R}^{d,d}$ using $\frac{1}{2}d(d-1)$ parameters. 
Note that the parametrization is not unique due to periodicity of goniometric functions. 
Also notice that for $\theta_{\cdot}=0,$ the matrix $\mathbf{U}$ is equal to identity.

\subsection{HOUSEHOLDER PARAMETERIZATION}
The second parametrization relies on the representation of unitary matrix $\mathbf{U} \in \mathbb{R}^{d,d}$ as a product of at most $d$ Householder transformations~\parencite{Urias2010}, i.e.~$\mathbf{U} = \mathbf{P}_d \mathbf{P}_{d-1} \ldots \mathbf{P}_1,$ where each $\mathbf{P}_i$ is defined by vector $\mathbf{y}_i$ as
\begin{equation}
    \mathbf{P}_i = \mathbf{I} - t_i \mathbf{y}_i \mathbf{y}_i^{\top},
\end{equation}
for $t_i = 2/\left\| \mathbf{{y}_i} \right\|^2$. 
Thus, by parametrizing the unitary matrix by vectors $\mathbf{y}_i,$ we can effectively generate a whole group of unitary matrices. 
Note that this construction over-parametrizes the group, as it is parametrized by $d^2$ parameters, where it has only $\frac{1}{2}d(d-1)$ degrees of freedom. 

\subsection{COMPUTATIONAL COMPLEXITY}
\textbf{Givens parametrization} has a computational complexity of $2d(d-1)$ multiplications and $d(d-1)$ additions during inference, and requires three times higher complexity of backpropagation, if intermediate results after individual givens rotations are not stored.
However, an on-the-fly computation from the output, which reduces the computational complexity, has been proposed in~\parencite{gomez2017reversible}. 
Without storing intermediate results, the memory requirements are negligible, as all operations are in-place.
    
\textbf{Householder parametrization} has a computational complexity of $2d^2$ multiplications and the same number of additions during inference. The complexity of backpropagation is again three times higher than that of the inference if we assume that intermediate results are not stored, but calculated from the output.

From the above, we see that Givens parametrization has lower computational complexity than that of the Householder. 
While the Houselder parametrization allows off-the-shelf automatic differentiation (AD), this has only recently been explored for Givens parametrizations \parencite{lezcano2019cheap}.
Therefore, we leverage the work by \parencite{lezcano2019cheap} to apply the Givens parametrization.

\section{RELATED WORK}
The proposed approach combines mixture models, probabilistic circuits, flow models, and representation of unitary matrices. Each of these topics has a rich literature, but the proposed combination is unique. Below, we review the most relevant work combining the transformation of variables and mixtures/PCs.

\textbf{Sum-Product Networks} 
The works by \parencite{tan2019hierarchical} and \parencite{Trapp2020} can be understood as flexible extension of SPN that use some kind of transformation in the leave nodes. In particular, \parencite{tan2019hierarchical} proposed to combine SPNs with variational autoencoders (VAE) on the leaves, while \parencite{Trapp2020} proposed to extend SPNs with Gaussian processes at leaves.
However, both approach do not exploit invertible transformations as internal nodes and are conceptually different to our proposal.

\textbf{Unitary matrices} 
The use of unitary matrices is beneficial for autoencoders~\parencite{tomczak2016improving}, convolution layers \parencite{putzky2019invert}, and in recurrent neural networks~\parencite{arjovsky2016unitary}.
However, to the best of our knowledge, they have not been applied to invertible flows as of now.

\textbf{Mixture models} Numerous extensions of mixture models share directions with our proposal. 
The use of mixture models on the latent layer of variational autoencoder~\parencite{dilokthanakul2016deep} may be understood as a transformation node as the root followed by a summing node. 
Mixtures of flow models have been recently suggested in~\parencite{papamakarios2019normalizing} using a shallow structure but without any experimental evidence. 
Further, optimal transport has been used within Gaussian mixture models in~\parencite{chen2018optimal}.

\section{EXPERIMENTS}
\begin{figure*}[t!]
\centering
\begin{subfigure}[t]{0.3\textwidth}
    \centering
    \includegraphics[height=0.2\textheight]{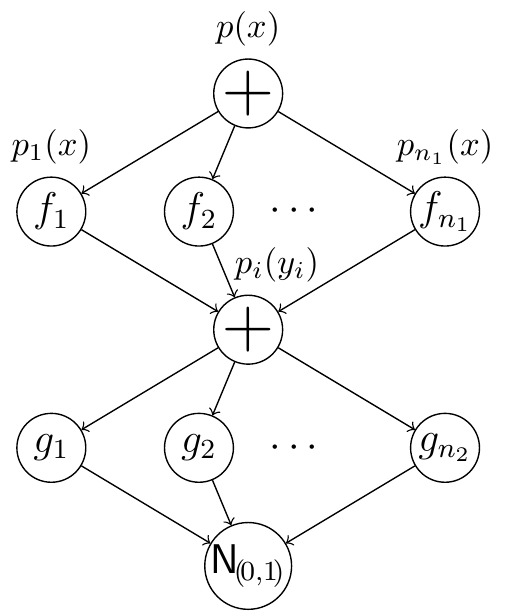}
    \caption{Sharing sum and transformation nodes}
\end{subfigure}%
~
\begin{subfigure}[t]{0.3\textwidth}
    \centering
    \includegraphics[height=0.2\textheight]{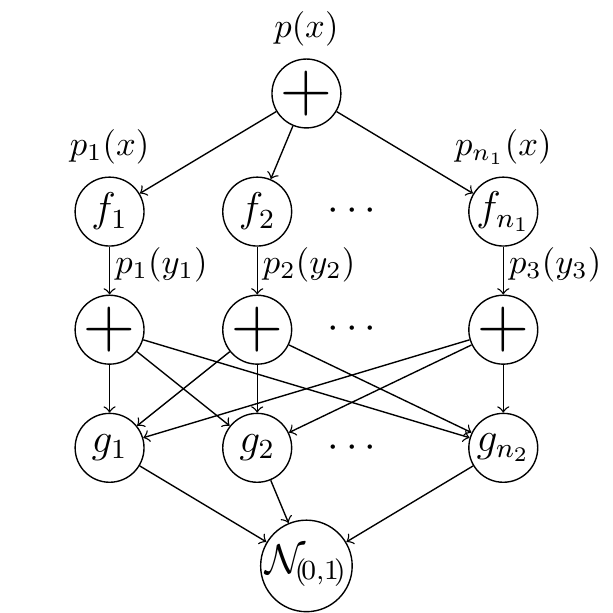}
    \caption{Sharing only transformation nodes}
\end{subfigure}%
~ 
\begin{subfigure}[t]{0.3\textwidth}
    \centering
    \includegraphics[height=0.2\textheight]{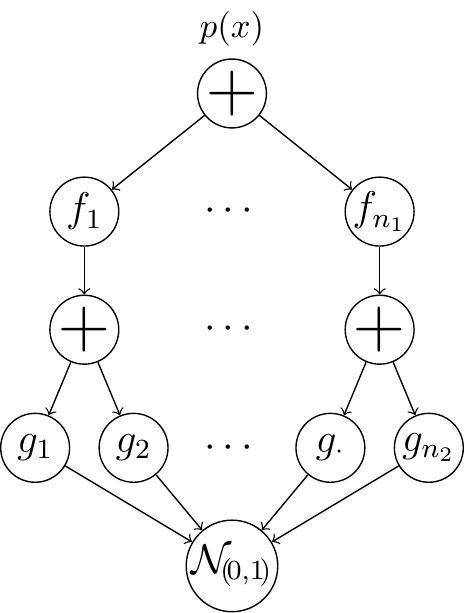}
    \caption{No sharing}
\end{subfigure}%
\caption{\label{fig:sharing} Different modes of sharing nodes (parameters) in SPTNs. 
The model in Subfigure (a) shares sum nodes $\oplus$ and transformation nodes $\{g_i\}_{i=1}^{n_2}$, which means that all $\{g_i\}_{i=1}^{n_1}$ have the the same child node $\oplus$; 
that in Subfigure (b) shares only transformation nodes $\{g_i\}_{i=1}^{n_2};$
and finally that in Subfigure (c) does not share any node except leaf, which does not have any parameters.}
\end{figure*}

As done in prior art, e.g.~\parencite{Poon11,Peharz2019RAT,vergari2019automatic}, we compared the performance of G-SPTNs for density estimation by maximizing the log-likelihood on the training set. Further, we additionally assessed their effectiveness in detecting anomalies. 
In the first experiment, G-SPTNs have been compared to SPNs, GMMs, and Masked Auto-regressive Flows~\parencite{papamakarios2017masked} (MAF). 
In the second experiment, we additionally compare agains Variational Autoencoders (VAE), Isolation Forests (IForest), and k-Nearest Neighbor (kNN) methods.


Experiments were carried out on 21 real-valued problems designed for anomaly detection~\parencite{pevny2016loda,vskvara2018generative}. These datasets are derived from the UCI database by following the approach of~\parencite{emmott2013systematic}, which converts multi-class datasets into binary classification task such that complexity of the classification problem is maximised. 
All models are trained on ``normal'' / majority class.
Respectively, the test log-likelihood is also reported on the majority class, while the AUC is estimated with respect to the conditional probabilities of both classes. 
All experiments were repeated five times with different random division of data into training (64\%), validation (16\%), and testing sets (20\%).

\subsection{ESTIMATING PARAMETERS}
Since SPTNs are a strict superset of SPNs, which are a strict superset of GMMs, we have used the same method to estimate the parameters of each models. 
In particular, we used stochastic gradient descend to maximize the log-likelihood as done in \parencite{Peharz2019RAT}. 
Since parameters of transformation nodes in SPTNs are differentiable, we can apply automatic differentiation to learn all parameters of SPTNs. 
Note that we used Adam~\parencite{kingma2014adam} with a batchsize of 100 in all experiments.

\subsection{COMPARED MODELS}
Since advanced structure learning in SPTN is not yet available, we have used a random sampling of architectures to learn both SPNs and SPTNs, akin to \cite{Peharz2019RAT,RashwanZP16}.
The best structure was selected on the validation set. 

\label{sub:models}
\textbf{Gaussian Mixture Model}
In case of GMMs, the only free architectural parameter is the number of components. We trained mixture models with $n \in \{2,4,8,16,\ldots,512\}$ components and full covariance Gaussian distributions implemented using an affine transformation before each leaf node with $\textsf{N}(0,\mathbf{I}).$ 

\textbf{Sum-Product network}
In case of SPNs, we have varied number of children of each sum node, $n \in \{2,4,8,16,\ldots,128\},$ the number of partitions under each product node, $b \in \{1,2,4,8,16,32\},$ and the number of layers, $l \in \{1, 2, 3, 4, 5\},$ where by layer we mean a combination of a sum and product node. 
Leaf nodes were fixed to $\textsf{N}(0,\mathbf{I}).$ 

\textbf{Affine Gaussian Sum-Product-Transform network}
To decrease the degrees of freedom in architecture search for SPTNs, we omitted product nodes in our architecture search. The sampled architectures had $l \in \{1, 2, 3\},$ layers (a layer is the combination of a sum node followed by an affine transformation node) and the number of children under each sum nodes, $n \in \{2,4,8,16\}.$ 
We also distinguished between architectures with \{no sharing, sharing of transformation nodes, sharing of sum and transformation nodes\} as outlined in Figure~\ref{fig:sharing}. Leaf nodes were again fixed to $\textsf{N}(0,\mathbf{I}).$ 

\textbf{Masked auto-regressive flows}
In case of MAFs, we performed a similar random search for the architecture as for G-SPTNs. 
We randomly sampled the number of masked auto-regressive layers~\parencite{germain2015made} $l\in \{5,10,20\},$ the number of layers in these layers $m \in \{1,2,3,4\},$ and the number of neurons in each layer $k\in\{10,20,40,80\}.$ The non-linearity was fixed to $\tanh,$ as was used in the accompanying material \url{https://github.com/gpapamak/maf}. 
Parameters of MAFs have been learned as described above, by maximising the log-likelihood~\parencite{papamakarios2017masked}. 
Similarly to all the above models, we performed 10,000 optimization steps.

\textbf{K-nearest neighbor} 
In case of k-NNs, we varied the number of nearest neighbors, $k$, and the anomaly score ($\kappa$, $\delta$, $\gamma$) as defined in~\parencite{harmeling2006outliers}). 

\textbf{Isolation forests}
In case of IForest, we varied the number of samples used to construct the individual trees ($256$, $512$, $1024$). 

\textbf{Variational auto-encoder} 
In case of VAE, we sampled the number of latent dimension from $\{2,5,10,20,40\},$ number of neurons in the hidden layers $\{10,20,40,80\},$ number of layers of the encoder and decoder,$\{1,2,3,4\}$, and the type of non-linearity, $\{\mathrm{tanh}, \mathrm{relu},\mathrm{leakyrelu}\}.$

SPTNs, SPNs, and GMMs were implemented using the same library available at \url{http://github.com}designed to implement arbitrarily DAGs containing sum, product, transformation nodes, and leaves represented by their density functions. 
All algorithms were trained for 10,000 iterations.
In the case of (G)-SPTN, MAF, and VAE we restricted the random search to 100 architectures or 3 days of total CPU time per problem and repetition of the problem. 

\subsection{DENSITY ESTIMATION}
\label{sub:density}
Table~\ref{tab:likelihood} shows the average test log-likelihood of G-SPTNs, SPNs, GMMs, and MAFs. 
The average log-likelihood for $\{x_i\}_{i=1}^n$ was calculated as $\frac{1}{n}\sum_{i=1}^n \log p(x_i)$ and reported values are macro averages over five repetitions of the experiment. 
According to results, on 15 out of 21 datasets, G-SPTNs obtain higher log-likelihood, and in some cases like miniboone, statlog-segment, and cardiotocography the improvement is significant. 
Contrary, the difference of SPTN to the best model in \texttt{Waveforms}, \texttt{Pendigits} (less than 0.1) and \texttt{Wine} is negligible. 
The only significant difference is on \texttt{wall-following-robot}. 
We conjecture this to be caused by the omission of product nodes in our architecture search. The poor performance of MAFs is caused by over-fitting, which can be seen from a high log-likelihood on training data (shown in Table~\ref{tab:trainlkl} in Appendix).

\begin{table}[t]
\begin{center}
\resizebox{\columnwidth}{!}{\begin{tabular}{lrrrr}
\hline\hline
\textbf{dataset} & \textbf{G-SPTN} & \textbf{SPN} & \textbf{GMM}& \textbf{MAF} \\ \hline
breast-cancer-wisconsin & \color{blue}{\textbf{-0.07}} & -20.55 & -6.05 & -1874.29 \\
cardiotocography & \color{blue}{\textbf{45.91}} & 11.06 & 10.95 & -598.63 \\
magic-telescope & -4.12 & -5.78 & -4.58 & \color{blue}{\textbf{-3.44}} \\
pendigits & -1.16 & -6.51 & -2.3 & \color{blue}{\textbf{1.21}} \\
pima-indians & \color{blue}{\textbf{-7.35}} & -8.18 & -8.7 & -68.81 \\
wall-following-robot & -12.59 & \color{blue}{\textbf{-4.45}} & -7.9 & -21.08 \\
waveform-1 & -23.87 & \color{blue}{\textbf{-23.85}} & -23.9 & -29.56 \\
waveform-2 & -23.91 & \color{blue}{\textbf{-23.85}} & -23.89 & -25.19 \\
yeast & \color{blue}{\textbf{8.22}} & -0.62 & -3.17 & 0.28 \\
ecoli & \color{blue}{\textbf{0.66}} & -3.21 & -3.79 & -3.93 \\
ionosphere & \color{blue}{\textbf{-11.75}} & -22.15 & -12.69 & -3457.46 \\
iris & \color{blue}{\textbf{-1.79}} & -2.28 & -1.87 & -53.97 \\
miniboone & \color{blue}{\textbf{162.46}} & 73.75 & 43.53 & -965573.45 \\
page-blocks & \color{blue}{\textbf{12.46}} & 2.58 & 3.75 & 5.67 \\
parkinsons & \color{blue}{\textbf{-3.55}} & -19.68 & -10.13 & -2931.57 \\
sonar & \color{blue}{\textbf{-74.8}} & -74.88 & -84.88 & -18991.33 \\
statlog-satimage & \color{blue}{\textbf{4.6}} & -9.65 & 2.52 & 4.1 \\
statlog-segment & \color{blue}{\textbf{34.39}} & 9.63 & 11.07 & -191.06 \\
statlog-vehicle & \color{blue}{\textbf{-2.76}} & -11.73 & -5.38 & -106.13 \\
synthetic-control-chart & \color{blue}{\textbf{-39.51}} & -43.92 & -40.21 & -9433.77 \\
wine & -13.61 & \color{blue}{\textbf{-13.39}} & -13.92 & -3074.69 \\
rank & \color{blue}{\textbf{1.38}} & 2.57 & 2.62 & 3.43 \\\hline\hline
\end{tabular}}
\end{center}
\caption{\label{tab:likelihood}
Average log-likelihood of the best models (higher is better) on the test set using five repetitions. Each best model was selected according to the performance on the validation set. 
The overall best performing model is highlighted in bold blue.
The average rank is calculated according to the ranking of each models on each problem (lower is better).
}
\end{table}

\textbf{Influence of parametrization of Unitary matrices}
Although Givens and Householder parametrizations generate the whole group of Unitary matrices, they might influence learning, for example, due to overparameterization in Householder or more natural representation of \emph{identity} in Givens. According to log-likelihoods of G-SPTNs of models with each parametrization (Table~\ref{tab:unitary} in Supplementary), it is impossible to state at the moment if there is any difference, as their average ranks overall problems was 1.52 (Givens) and 1.48 (Householder).

\textbf{Influence of node sharing}
Similar to SPNs, SPTN allow flexible sharing of nodes within the network. 
We have compared three cases: no sharing, sharing transformation nodes, and sharing sum and transformation nodes, which are outlined in Figure~\ref{fig:sharing}. According to log-likelihood shown in Appendix in Table~\ref{tab:sharing}, models sharing sum and transformation nodes are inferior to models without sharing and to those sharing only transformation nodes. 
This is somewhat surprising, since the number of parameters of models sharing only transformation nodes is similar to those sharing sum and transformation nodes.
We assume that the significantly lower performance is due to overfitting.

\textbf{Influence of (non)-linearity}
Transformation nodes in SPTN permit non-linear functions after the affine transformation (at the expense of losing tractable marginalization). 
To observe, if non-linear transformations improves the fit, we have compared SPTN with linear, leaky-relu~\parencite{maas2013rectifier}, and selu~\parencite{klambauer2017self} transformations applied element-wise after affine transformation in transformation nodes. 
According to the log-likelihood scores reported in Table~\ref{tab:function} (in Appendix), G-SPTN with linear functions perform better than SPTNs with selu and leaky relu, as their average rank over problems was $1.12,$ $3.0,$ and $1.88.$ respectively. 
Therefore, it seems that affine transformations are sufficient (and potentially prevent overfitting) for the respective datasets.

\subsection{ANOMALY DETECTION}
\begin{table}[t]
\resizebox{\columnwidth}{!}{\begin{tabular}{lrrrrrrr}
\hline\hline
\textbf{dataset} & \textbf{G-SPTN} & \textbf{SPN} & \textbf{GMM} & \textbf{k-NN} & \textbf{VAE} & \textbf{IForest} & \textbf{MAF}\\
breast-cancer-wisconsin & \color{blue}{\textbf{0.97}} & \color{blue}{\textbf{0.97}} & 0.95 & 0.94 & \color{blue}{\textbf{0.97}} & \color{blue}{\textbf{0.97}} & 0.88 \\
cardiotocography & 0.75 & 0.72 & 0.56 & 0.81 & \color{blue}{\textbf{0.84}} & 0.7 & 0.63 \\
magic-telescope & 0.97 & 0.95 & 0.96 & 0.96 & 0.85 & 0.89 & \color{blue}{\textbf{0.98}} \\
pendigits & \color{blue}{\textbf{0.99}} & 0.98 & \color{blue}{\textbf{0.99}} & \color{blue}{\textbf{0.99}} & 0.85 & 0.94 & 0.98 \\
pima-indians & 0.89 & 0.88 & \color{blue}{\textbf{0.91}} & \color{blue}{\textbf{0.91}} & 0.89 & 0.9 & 0.78 \\
wall-following-robot & 0.84 & \color{blue}{\textbf{0.87}} & 0.83 & 0.84 & 0.68 & 0.76 & \color{blue}{\textbf{0.87}} \\
waveform-1 & 0.82 & 0.84 & 0.79 & 0.82 & \color{blue}{\textbf{0.87}} & 0.82 & 0.68 \\
waveform-2 & 0.84 & 0.84 & 0.81 & 0.83 & \color{blue}{\textbf{0.87}} & 0.83 & 0.71 \\
yeast & 0.78 & 0.75 & 0.79 & 0.76 & \color{blue}{\textbf{0.82}} & 0.67 & 0.78 \\
ecoli & 0.9 & 0.9 & \color{blue}{\textbf{0.91}} & \color{blue}{\textbf{0.91}} & 0.81 & 0.82 & 0.86 \\
ionosphere & 0.92 & \color{blue}{\textbf{0.99}} & 0.9 & 0.95 & 0.9 & 0.89 & 0.78 \\
iris & \color{blue}{\textbf{0.97}} & 0.96 & 0.96 & 0.92 & 0.71 & 0.9 & 0.86 \\
miniboone & 0.88 & 0.87 & 0.88 & 0.86 & \color{blue}{\textbf{0.94}} & 0.84 & 0.76 \\
page-blocks & \color{blue}{\textbf{0.98}} & \color{blue}{\textbf{0.98}} & \color{blue}{\textbf{0.98}} & \color{blue}{\textbf{0.98}} & \color{blue}{\textbf{0.98}} & 0.97 & \color{blue}{\textbf{0.98}} \\
parkinsons & 0.77 & 0.67 & 0.77 & 0.81 & \color{blue}{\textbf{0.87}} & 0.75 & 0.71 \\
sonar & 0.59 & 0.68 & 0.65 & 0.7 & \color{blue}{\textbf{0.79}} & 0.66 & 0.55 \\
statlog-satimage & 0.86 & 0.97 & 0.8 & \color{blue}{\textbf{0.98}} & 0.96 & 0.94 & 0.89 \\
statlog-segment & 0.82 & \color{blue}{\textbf{0.87}} & 0.84 & 0.85 & 0.64 & 0.67 & 0.84 \\
statlog-vehicle & \color{blue}{\textbf{0.79}} & 0.71 & \color{blue}{\textbf{0.79}} & 0.76 & 0.75 & 0.73 & 0.66 \\
synthetic-control-chart & 0.88 & \color{blue}{\textbf{0.98}} & 0.87 & \color{blue}{\textbf{0.98}} & 0.93 & 0.91 & 0.71 \\
wine & \color{blue}{\textbf{0.98}} & \color{blue}{\textbf{0.98}} & \color{blue}{\textbf{0.98}} & \color{blue}{\textbf{0.98}} & 0.91 & 0.89 & 0.86 \\
rank & 2.7 & 3.0 & 3.48 & \color{blue}{\textbf{2.57}} & 3.57 & 5.0 & 5.29 \\\hline\hline
\end{tabular}}
\caption{\label{tab:anomaly} Average Area Under the ROC Curve (AUC) for each model calculated on the test set using five repetitions (higher is better). The best model for each approach was again selected on the validation set. The best model is highlighted in bold blue.}
\end{table}

As shown in~\parencite{vergari2019automatic,Peharz2019RAT}, SPNs can be used to effectively detect anomalies. 
Therefore, we assessed G-SPTNs, SPNs, GMMs, VAEs, IForests, and k-NNs as anomaly detectors. 
In case of G-SPTNs, SPNs, and GMMs we used the negative log-likelihood as an anomaly score as proposed in~\parencite{vergari2019automatic}. 
Note that k-NN are known to obtain good performance on this task~\parencite{pevny2016loda,vskvara2018generative}.

The quality of anomaly detection was measured using the Area Under the ROC Curve (AUC), which is the standard within the field of anomaly detection. As above, the hyperparameters/architecture was selected according to AUC estimated on the validation set,\footnote{We admit that selecting models based on their performance on validation dataset implies that few anomalies are available, which is in a slight disagreement with a typical assumption in the field that anomalies are not available during training. Since the problem of model selection is still unresolved in the field of anomaly detection, we do not aim to solve it here, and use few anomalies for this.} while the reported values are an average over five repetitions. 
Note that the AUC was estimated from the ``normal'' class and that \emph{easy} anomalies~\parencite{emmott2013systematic,pevny2016loda} as more difficult anomalies are not anomalies in the sense that they are not located in a region of the low density of the normal class.

Table~\ref{tab:anomaly} shows average AUCs of the compared models on different datasets. 
The overall best method to identify anomalies is k-NN, which has been previously reported to obtain very competitive results \parencite{vskvara2018generative,pevny2016loda}.
The proposed G-SPTN scored second with SPN being the third. To our surprise, VAEs frequently considered as a modern state of the art performed inferior to SPNs (and also to k-NN as has been already reported in~\parencite{vskvara2018generative}). 

\section{CONCLUSION}
In this paper we suggest to extended the compositions used in Probabilistic Circuits to additionally include invertible transformations. Within this new class, called Sum-Product-Transform Networks (SPTN), two frameworks, Probabilistic Circuits, and Flow models, unite and each becomes a special case. Since models in SPTNs, in general, do not support efficient marginalization and conditioning, an important sub-class (called G-SPTN) for which these operations are efficient was identified. G-SPTN restrict transformations to be affine and leaf nodes to be Gaussian distributions. The affine transformations keep their projection matrices in SVD forms, which is facilitated by parametrizing groups of unitary matrices, which is treated in detail. 

The proposed approach was experimentally compared to Sum-Product Networks (SPNs), Gaussian mixture models, and Masked autoregressive flows on a corpus of 21 publicly available problems. Because SPTNs unify flow models and SPNs, it should not be surprising that the results confirm their good modeling properties. 
But importantly, this good performance was achieved by G-SPTN, which still feature efficient marginalization and conditioning. 

Despite good experimental results, there remain several open problems some of which we plan to address in the future. 
Specifically, a major challenge in learning SPNs is structure learning, which has inspired many sophisticated techniques, e.g.~\cite{vergari2015simplifying,Peharz2019RAT,Trapp2019}. 
Learning structures for SPTNs is even more challenging and we hope that some of the existing technique for SPNs can be extended to SPTNs in the future.
Moreover, we want to explore more efficient parameter learning for SPTN, as done in the SPN literature, and conduct a more in depth investigation of the capacities of SPTNs for anomaly detection.

\printbibliography

\appendix

\begin{table}
\begin{tabular}{lrrr}
\hline\hline
\textbf{dataset} & \textbf{id.} & \textbf{lrelu} & \textbf{selu} \\\hline
breast-cancer-wisconsin & -0.07 & -89.18 & \color{blue}{\textbf{0.04}} \\
magic-telescope & \color{blue}{\textbf{-4.12}} & -28.47 & -5.73 \\
pendigits & \color{blue}{\textbf{-1.16}} & -41.54 & -3.45 \\
pima-indians & \color{blue}{\textbf{-7.35}} & -21.92 & -9.48 \\
wall-following-robot & \color{blue}{\textbf{-12.59}} & -61.55 & -14.67 \\
waveform-1 & \color{blue}{\textbf{-23.87}} & -60.64 & -25.84 \\
waveform-2 & \color{blue}{\textbf{-23.91}} & -59.5 & -25.88 \\
yeast & \color{blue}{\textbf{8.22}} & -28.4 & -4.28 \\
ecoli & \color{blue}{\textbf{0.66}} & -18.71 & -3.31 \\
ionosphere & -11.75 & -104.7 & \color{blue}{\textbf{-10.66}} \\
iris & \color{blue}{\textbf{-1.79}} & -12.26 & -2.34 \\
parkinsons & \color{blue}{\textbf{-3.55}} & -66.73 & -6.77 \\
sonar & \color{blue}{\textbf{-74.8}} & -221.28 & -78.78 \\
statlog-satimage & \color{blue}{\textbf{4.6}} & -118.23 & 0.98 \\
statlog-vehicle & \color{blue}{\textbf{-2.76}} & -55.62 & -4.25 \\
synthetic-control-chart & \color{blue}{\textbf{-39.51}} & -197.09 & -44.59 \\
wine & \color{blue}{\textbf{-13.61}} & -60.51 & -14.84 \\
rank & \color{blue}{\textbf{1.12}} & 3.0 & 1.88 \\\hline\hline
\end{tabular}
\caption{\label{tab:function}
Average log-likelihood on test set from five repetitions of an experiment  (higher is better) of SPTN models with different functions, namely identity (caption ``id.''), leaky-relu (caption ``leaky rely''), and selu, in transformation nodes. The best model for each combination of problem and cross-validation fold was selected according to log-likelihood on validation data. Average rank is calculated from ranks of models on each problem (lower is better). The best model is highlighed in bold blue. }
\end{table}

\label{app:parametrization}
\begin{table}
\begin{center}
\begin{tabular}{rrr}
\hline\hline
\textbf{dataset} & \textbf{givens} & \textbf{householder} \\
\hline
breast-cancer-wisconsin & -0.46 & \color{blue}{\textbf{0.08}} \\
cardiotocography & 39.23 & \color{blue}{\textbf{45.91}} \\
magic-telescope & \color{blue}{\textbf{-4.12}} & -4.55 \\
pendigits & \color{blue}{\textbf{-1.26}} & -1.52 \\
pima-indians & -9.05 & \color{blue}{\textbf{-7.35}} \\
wall-following-robot & \color{blue}{\textbf{-12.66}} & -13.93 \\
waveform-1 & -23.9 & \color{blue}{\textbf{-23.88}} \\
waveform-2 & -23.98 & \color{blue}{\textbf{-23.92}} \\
yeast & 4.5 & \color{blue}{\textbf{8.22}} \\
ecoli & -3.4 & \color{blue}{\textbf{0.66}} \\
ionosphere & \color{blue}{\textbf{-10.48}} & -13.04 \\
iris & -1.98 & \color{blue}{\textbf{-1.79}} \\
miniboone & \color{blue}{\textbf{162.46}} & 112.08 \\
page-blocks & \color{blue}{\textbf{12.53}} & 11.62 \\
parkinsons & -8.91 & \color{blue}{\textbf{-3.55}} \\
sonar & \color{blue}{\textbf{-73.97}} & -87.1 \\
statlog-satimage & \color{blue}{\textbf{4.55}} & 3.76 \\
statlog-segment & 31.87 & \color{blue}{\textbf{33.89}} \\
statlog-vehicle & \color{blue}{\textbf{-2.82}} & -3.35 \\
synthetic-control-chart & \color{blue}{\textbf{-39.51}} & -40.3 \\
wine & -18.38 & \color{blue}{\textbf{-13.61}} \\ \hline 
average rank & 1.52 & \color{blue}{\textbf{1.47}} \\\hline\hline
\end{tabular}
\end{center}
\caption{\label{tab:unitary}Average log-likelihood  of G-SPTN models with affine Tranformation nodes realized either by set of \emph{Givens} rotations or by a set of \emph{Householder} rotations.  Average log-likelihood is from  test sets of  five repetitions of an experiment. The best model for each combination of problem and cross-validation fold was selected according to log-likelihood on validation data. The best model is highlighed in bold blue.}
\end{table}

\begin{table}[t]
\begin{center}
\begin{tabular}{lrrr}
\hline\hline
\textbf{dataset} & \textbf{none} & \textbf{trans.} & \textbf{all} \\\hline
breast-cancer-wisconsin & -0.36 & \color{blue}{\textbf{0.33}} & 0.03 \\
cardiotocography & \color{blue}{\textbf{44.13}} & 41.22 & 38.7 \\
magic-telescope & \color{blue}{\textbf{-4.12}} & -4.47 & -4.79 \\
pendigits & -1.54 & \color{blue}{\textbf{-1.34}} & -1.57 \\
pima-indians & -8.18 & \color{blue}{\textbf{-7.75}} & -8.1 \\
wall-following-robot & -13.43 & -13.56 & \color{blue}{\textbf{-12.64}} \\
waveform-1 & \color{blue}{\textbf{-23.87}} & \color{blue}{\textbf{-23.87}} & -23.9 \\
waveform-2 & -23.98 & \color{blue}{\textbf{-23.91}} & -23.92 \\
yeast & 7.56 & 6.77 & \color{blue}{\textbf{7.64}} \\
ecoli & 0.14 & -0.36 & \color{blue}{\textbf{0.47}} \\
ionosphere & \color{blue}{\textbf{-11.57}} & -12.19 & -14.59 \\
iris & -1.79 & -1.74 & \color{blue}{\textbf{-1.73}} \\
miniboone & \color{blue}{\textbf{162.46}} & 129.08 & 81.62 \\
page-blocks & 11.62 & \color{blue}{\textbf{12.53}} & 11.71 \\
parkinsons & -3.95 & \color{blue}{\textbf{-3.58}} & -4.42 \\
sonar & -82.23 & -79.1 & \color{blue}{\textbf{-78.21}} \\
statlog-satimage & \color{blue}{\textbf{4.65}} & 4.22 & 3.91 \\
statlog-segment & \color{blue}{\textbf{34.58}} & 32.92 & 33.31 \\
statlog-vehicle & \color{blue}{\textbf{-2.82}} & -3.72 & -3.58 \\
synthetic-control-chart & -40.3 & \color{blue}{\textbf{-39.51}} & -40.68 \\
wine & \color{blue}{\textbf{-13.61}} & \color{blue}{\textbf{-13.61}} & \color{blue}{\textbf{-13.61}} \\
rank & 1.86& \color{blue}{\textbf{1.80}} & 2.14 \\\hline\hline
\end{tabular}
\end{center}
\caption{\label{tab:sharing}Average log-likelihood  of G-SPTN models without any sharing of nodes captioned ``none'' (see Figure~\ref{fig:sharing}(c), sharing Tranformation nodes only captioned ``Affine''  (see Figure~\ref{fig:sharing}(b), and sharing sum and Tranformation nodes captioned all  (see Figure~\ref{fig:sharing}(a). Average log-likelihood is from  test sets of  five repetitions of an experiment. The best model for each combination of problem and cross-validation fold was selected according to log-likelihood on validation data. The best model is highlighed in bold blue.}
\end{table}

\begin{table}[t]
    \centering
    \resizebox{\columnwidth}{!}{\begin{tabular}{rrrrr}
\hline\hline
\textbf{dataset} & \textbf{G-SPTN} & \textbf{SPN} & \textbf{GMM} & \textbf{MAF} \\
\hline
breast-cancer-wisconsin & \color{blue}{\textbf{65.18}} & -13.4 & 10.6 & 62.91 \\
cardiotocography & 54.86 & 11.98 & 12.69 & \color{blue}{\textbf{59.52}} \\
magic-telescope & -3.24 & -5.68 & -3.96 & \color{blue}{\textbf{-2.72}} \\
pendigits & 4.84 & -6.13 & -0.88 & \color{blue}{\textbf{7.71}} \\
pima-indians & \color{blue}{\textbf{13.9}} & -5.55 & 0.59 & 1.1 \\
wall-following-robot & 6.13 & -3.56 & -0.31 & \color{blue}{\textbf{15.28}} \\
waveform-1 & \color{blue}{\textbf{-1.07}} & -23.54 & -9.29 & -9.25 \\
waveform-2 & \color{blue}{\textbf{1.6}} & -23.6 & -9.16 & -10.75 \\
yeast & \color{blue}{\textbf{13.52}} & 0.66 & 0.37 & 11.08 \\
ecoli & \color{blue}{\textbf{12.28}} & -1.18 & 1.35 & 4.71 \\
ionosphere & 34.79 & -7.1 & 6.32 & \color{blue}{\textbf{85.17}} \\
iris & \color{blue}{\textbf{11.03}} & -0.71 & 0.96 & 1.45 \\
miniboone & 162.87 & 73.71 & 43.57 & \color{blue}{\textbf{181.29}} \\
page-blocks & 13.33 & 2.89 & 4.09 & \color{blue}{\textbf{21.38}} \\
parkinsons & 22.34 & -5.87 & 19.94 & \color{blue}{\textbf{58.29}} \\
sonar & 19.54 & -30.95 & 1.01 & \color{blue}{\textbf{162.87}} \\
statlog-satimage & 5.94 & -9.27 & 2.84 & \color{blue}{\textbf{21.07}} \\
statlog-segment & 47.33 & 11.51 & 13.83 & \color{blue}{\textbf{50.28}} \\
statlog-vehicle & 9.81 & -8.64 & 4.3 & \color{blue}{\textbf{21.59}} \\
synthetic-control-chart & -26.27 & -41.04 & -32.65 & \color{blue}{\textbf{106.74}} \\
wine & \color{blue}{\textbf{38.01}} & -4.8 & 11.97 & 22.35 \\\hline\hline
\end{tabular}}
\caption{
    \label{tab:trainlkl}
    Average log-likelihood of models (higher is better) on training set  from five repetitions of an experiment. The best model for each combination of problem and cross-validation fold was selected according to log-likelihood on training set, as the purpose of this Table is to show the over-fitting of MAF.}
\end{table}

\end{document}